\pdfoutput=1

\documentclass{article}

\usepackage{hyperref}

\usepackage{textcomp}

\usepackage[utf8]{inputenc} 
\usepackage[T1]{fontenc}    
\usepackage{url}            
\usepackage{amsfonts}       
\usepackage{nicefrac}       
\usepackage{ifthen}
\usepackage{bbm}
\usepackage{amssymb}
\usepackage{algorithmicx}
\usepackage[cmex10]{amsmath}
\usepackage{amsthm}
\usepackage{mathtools}
\usepackage{bm}
\usepackage[inline, shortlabels]{enumitem}
\usepackage[electronic]{ifsym}
\usepackage{physics}
\usepackage{subcaption}

\usepackage{multirow}
\usepackage{makecell,pbox,ragged2e,hhline}
\usepackage[table, x11names, svgnames]{xcolor}
\usepackage{color,diagbox,tabularx}

\usepackage{tikz}
\usetikzlibrary{fit, shapes.geometric, shapes.misc, arrows.meta, positioning, decorations.markings, matrix}
\tikzset{>={Latex[width=1.2mm,length=1.5mm]}}
\usepackage{smartdiagram}
\usesmartdiagramlibrary{additions}
\usepackage{pgfplots}
\pgfplotsset{compat=newest, ticks=none}
\usepackage{varwidth}
\tikzset{circ2/.style = {circle, draw=black!100, fill=yellow!25, thin, minimum height=5mm, minimum width=8mm, 
  inner sep=1.2mm},
  rect2/.style = {rectangle, draw=black!100, fill=cyan!18, thin, minimum height=7.5mm, minimum width=11mm, 
  inner sep=1.5mm},
  circ/.style = {circle, draw=black!100, fill=red!35, thin, minimum size=7mm},
  outer/.style = {rounded corners=0.2cm, draw=black!100, dashed, inner sep = 3mm},
  outer2/.style = {rounded corners=0.2cm, draw=black!100, dashed, inner sep = 3mm},
  encode/.style = {trapezium, draw=black!100, fill=green!20, trapezium angle=75, shape border rotate=90, shift = {(-0.75, 0)}, minimum width=5mm, minimum height=11mm, align=center},
  decode/.style = {trapezium, draw=black!100, fill=green!20, trapezium angle=-75, shape border rotate=90, minimum width=5mm, minimum height=11mm, align=center},
  rect3/.style = {rectangle, draw=black!100, fill=cyan!1, thin, minimum height=7.5mm, minimum width=11mm, 
  inner sep=1.5mm},
  bluelines/.style={smooth, very thick, blue!40!gray}
  }
\tikzset{every picture/.style={/utils/exec={\fontfamily{lmss}}}}
\pgfmathdeclarefunction{gauss}{2}{%
  \pgfmathparse{1/(#2*sqrt(2*pi))*exp(-((x-#1)^2)/(2*#2^2))}%
}

\usepackage{steinmetz}
\usepackage{xpatch}
\xpatchcmd{\phase}{#2}{\hspace{0.8pt}\vphantom{\scalebox{0.8}{\tiny{,}}}#2\hspace{1.4pt}}{}{}

\makeatletter
\newdimen\@widthOfTo%
\newdimen\@widthOfImplies%
\pgfmathsetmacro{\@scaleFactorImplies}{\@widthOfTo/\@widthOfImplies}%
\newcommand*{\ScaledImplies}{\mathrel{\raisebox{0.3ex}{\scalebox{\@scaleFactorImplies}{\ensuremath{\Longrightarrow}}}}}%
\makeatother

\newcolumntype{x}[1]{>{\centering\arraybackslash}p{#1}}

\usepackage{cellspace}
\renewcommand\arraystretch{1.2}

\usepackage{colortbl}
\definecolor{LightCyan}{rgb}{0.75,1,1}
\definecolor{LightGrey}{rgb}{0.95,0.95,0.95}
\definecolor{DarkBlue}{rgb}{0,0,0.4}
\definecolor{Yellow}{rgb}{1,1,0.6}

\DeclareMathOperator*{\sign}{sgn}

\let\originalleft\left
\let\originalright\right
\renewcommand{\left}{\mathopen{}\mathclose\bgroup\originalleft}
\renewcommand{\right}{\aftergroup\egroup\originalright}

\newcommand{\bw}{{\bm w}}
\newcommand{\bx}{{\bm x}}
\newcommand{\be}{{\bm e}}
\newcommand{\by}{{\bm y}}
\newcommand{\ba}{{\bm a}}

\newcommand{\bz}{{\bm z}}

\usepackage{listings}
\lstdefinestyle{python}{
  belowcaptionskip=1\baselineskip,
  breaklines=true,
  frame=shadowbox,
  rulesepcolor=\color{gray},
  xleftmargin=\parindent,
  language=Python,
  showstringspaces=false,
  basicstyle=\footnotesize\ttfamily,
  keywordstyle=\bfseries\color{deepblue},
  moredelim=**[s][\color{blue}]{'''}{'''},
  commentstyle=\itshape\color{magenta},
  identifierstyle=\color{black},
  stringstyle=\color{red}
}
\lstdefinestyle{output}{
  belowcaptionskip=1\baselineskip,
  breaklines=true,
  frame=L,
  basicstyle=\footnotesize\ttfamily,
  xleftmargin=\parindent
}

\theoremstyle{plain}

\newtheorem{proposition}{Proposition}

\theoremstyle{definition}

\theoremstyle{remark}

\usepackage[noadjust]{cite}

\usepackage{spconf,amsmath,graphicx}

\usepackage{enumitem}
\usepackage{amsmath,amsfonts,amssymb,amsthm}
\usepackage{mathtools}
\usepackage{commath}
\usepackage{diagbox}
\usepackage{bm}
\usepackage{tabu}
\usepackage{booktabs}
\usepackage{subcaption}
\usepackage{lipsum}
\usepackage{environ}
\usepackage{stfloats}
\usepackage{graphicx}
\usepackage{tikz}
\usetikzlibrary{fit, shapes.geometric, shapes.misc, arrows.meta, positioning, decorations.markings, matrix}
\tikzset{>={Latex[width=1.2mm,length=1.5mm]}}
\usepackage{smartdiagram}
\usesmartdiagramlibrary{additions}
\usepackage{pgfplots}
\pgfplotsset{compat=newest, ticks=none}
\usepackage{varwidth}
\usetikzlibrary{arrows,positioning, shapes.symbols,shapes.callouts,patterns}

\tikzset{circ2/.style = {circle, draw=black!100, fill=yellow!25, thin, minimum height=5mm, minimum width=8mm, 
  inner sep=1.2mm},
  rect2/.style = {rectangle, draw=black!100, fill=cyan!18, thin, minimum height=7.5mm, minimum width=11mm, 
  inner sep=1mm},
  circ/.style = {circle, draw=black!100, thin, minimum size=7mm},
  outer/.style = {rounded corners=0.2cm, draw=black!100, dashed, inner sep = 3mm},
  encode/.style = {trapezium, draw=black!100, fill=green!20, trapezium angle=75, shape border rotate=90, shift = {(-0.75, 0)}, minimum width=5mm, minimum height=11mm, align=center},
  decode/.style = {trapezium, draw=black!100, fill=green!20, trapezium angle=-75, shape border rotate=90, minimum width=5mm, minimum height=11mm, align=center},
  rect/.style = {rectangle, draw=black!100, fill=yellow!20, thin, minimum height=7.5mm, minimum width=5mm, 
  inner sep=1.5mm},
  rect_filter/.style = {rectangle, draw=black!100, fill=yellow!20, thin, minimum height=7.5mm, minimum width=5mm, 
  inner sep=0.1mm},
  bluelines/.style={smooth, very thick, blue!40!gray}
  }
\tikzset{every picture/.style={/utils/exec={\fontfamily{lmss}}}}

\title{Polarizing Front Ends for Robust CNNs}

\name{Can Bakiskan \sthanks{Corresponding author: \texttt{canbakiskan@ucsb.edu}}  \quad Soorya Gopalakrishnan \quad Metehan Cekic \quad Upamanyu Madhow \quad Ramtin Pedarsani}
\address{University of California, Santa Barbara, Department of Electrical and Computer Engineering}

\begin{document}

\maketitle

\begin{keywords}
adversarial machine learning, quantization, front-end defense
\end{keywords}

\begin{abstract}
The vulnerability of deep neural networks to small, adversarially designed perturbations can be attributed to their ``excessive linearity.'' In this paper, we propose
a bottom-up strategy for attenuating adversarial perturbations using a nonlinear front end which polarizes and quantizes the data.  We observe that ideal
polarization can be utilized to completely eliminate perturbations, develop algorithms to learn approximately polarizing bases for data, and investigate the effectiveness of the proposed strategy on the MNIST and Fashion MNIST datasets.  
\end{abstract}
\section{Introduction}
\label{sec:intro}

Given the immense impact of deep learning on a diversity of fields, its vulnerability
to tiny {\it adversarial} perturbations \cite{wildpattern2018, szegedy2013intriguing} is of great concern.
For image datasets, for example, such perturbations are almost imperceptible for humans, but they can render state-of-the-art models useless, causing misclassification with high confidence.  State of the art adversarial attacks are variants of gradient ascent, utilizing the local linearity of deep networks.
State of the art defenses are based on adversarial training, using training examples obtained using adversarial attacks, but yield little insight
into, or guarantees of, the achieved robustness.

In this paper, we investigate a systematic, bottom-up approach to robustness, studying a defense based on a nonlinear front end for attenuating adversarial
perturbations before they reach the deep network.  We focus on $\ell_{\infty}$-bounded perturbations. Our approach consists of 
{\it polarizing} the input data into well-separated clusters by projecting onto an appropriately selected basis (implemented using convolutional filters), and then
quantizing the output using thresholds that scale with the $\ell_1$ norm of the basis functions.  For ideal polarization, we prove that perturbations
are completely eliminated.  We introduce a regularization technique to learn polarizing bases from data, and demonstrate the efficacy of the proposed
defense for the MNIST and Fashion MNIST datasets. 

\section{Background} \label{sec:background}

\begin{figure*}[!t]
\begin{subfigure}[b]{0.33\linewidth}
\centering
	\scalebox{0.9}{
	\begin{tikzpicture}
		\begin{axis}[
			axis lines = center,
			set layers=standard,
			xmin=-6, xmax=6, ymin=-6, ymax=6,
			y=0.5cm/1.4,
			x=0.5cm*1,
			xlabel = {$a$},
			ylabel = {$f(a)$},
			x label style= {at ={(axis cs: 6, -1.6)}},
			yticklabels={,,},
			]
			\addplot[domain=2:6, bluelines]{x};
			\addplot[domain=-6:-2, bluelines]{x};
			\addplot[domain=-2:2, bluelines]{0};
			\draw[dotted, bluelines] (axis cs:2,0) -- (axis cs:2,2);
			\draw[dotted, bluelines] (axis cs:-2,0) -- (axis cs:-2,-2);
			\begin{pgfonlayer}{axis background}
				\filldraw[fill=red!6, draw=none]
								(axis cs: -6, -6) rectangle (axis cs: -2, 6);
				\filldraw[fill=red!6, draw=none]
								(axis cs: 2, -6) rectangle (axis cs: 6, 6);
			\end{pgfonlayer}
			\node at (axis cs: 2.2, -1) () { $\epsilon \Vert\bw\Vert_1$};
			\node at (axis cs: -2.2, 1) () { $-\epsilon \Vert\bw\Vert_1$};
			\node[fill=white, inner sep = -3pt, minimum height = 42pt] at (axis cs: 0, -3.5) () {\scriptsize
				\begin{tabular}{c}
				Perturbation \\ removed \\ (for a majority\\ of neurons)
				\end{tabular}
				};
			\node at (axis cs: 4.4, 1.6) () {\scriptsize
				\begin{tabular}{c}
				Perturbation \\ rides on top \\ of signal
				\end{tabular}
				};
		\end{axis}
	\end{tikzpicture}}
	\caption{}
	\label{fig:sparse_act}
\end{subfigure}
\hfill
\begin{subfigure}[b]{0.33\linewidth}
\centering
	\scalebox{0.93}{
	\begin{tikzpicture}
		\begin{axis}[
			axis lines = center,
			set layers=standard,
			axis on top = true,
			xmin=-10, xmax=10, ymin=-0.5, ymax=0.5,
			y=4cm,
			x=0.27cm,
			xlabel = {\small $\displaystyle{\phantom{a}\frac{a}{\Vert\bw\Vert_1}}$},
			x label style= {at ={(axis cs: 11.5, -0.25)}},
			yticklabels={,,},
			]
			\addplot[domain=-3:3, fill=gray!30, color=gray!30] {0.8*gauss(0,0.7)};
			\addplot[domain=7:10, fill=gray!30, color=gray!30] {0.3*gauss(8.5,0.4)};
			\addplot[domain=-10:-7, fill=gray!30, color=gray!30] {0.3*gauss(-8.5,0.4)};
			\addplot[domain=-5:5, bluelines]{0};
			\addplot[domain=5:10, bluelines]{0.4};
			\addplot[domain=-10:-5, bluelines]{-0.4};
			\draw[bluelines] (axis cs:5,0) -- (axis cs:5,0.4);
			\draw[bluelines] (axis cs:-5,0) -- (axis cs:-5,-0.4);
			\node[red!60!black] at (axis cs: 4, -0.1) () { $\epsilon$};
			\node[red!60!black] at (axis cs: 6, -0.1) () { $\epsilon$};
			\node[red!60!black] at (axis cs: 5, -0.22) () {\scriptsize
				\begin{tabular}{c}
				Danger \\[-2pt] zone
				\end{tabular}
				};
			\draw[<->, red!60!black] (axis cs: -3, -0.05) -- (axis cs: -4.95, -0.05);
			\draw[<->, red!60!black] (axis cs: -5.05, -0.05) -- (axis cs: -7, -0.05);
			\node[red!60!black] at (axis cs: -4, -0.1) () { $\epsilon$};
			\node[red!60!black] at (axis cs: -6, -0.1) () { $\epsilon$};
			\draw[<->, red!60!black] (axis cs: 3, -0.05) -- (axis cs: 4.95, -0.05);
			\draw[<->, red!60!black] (axis cs: 5.05, -0.05) -- (axis cs: 7, -0.05);
			\node[] at (axis cs: -5, 0.4) () {\scriptsize
				\begin{tabular}{c}
				Polarized histogram \\of $\bw^T\bx/\Vert\bw\Vert_1$
				\end{tabular}
				};
			\draw[->] (axis cs: -2.5, 0.32) -- (axis cs: -1, 0.25);
			\draw[->] (axis cs: -7, 0.32) -- (axis cs: -8, 0.25);
			\node[] at (axis cs: 3.23, 0.4) () {\small $f(a)$};
			\draw[->] (axis cs: 3.9, 0.35) -- (axis cs: 4.8, 0.3);
			\node[blue!60!black] at (axis cs: 9.6, 0.45) () {$1$};
			\node[blue!60!black] at (axis cs: 0.5, -0.05) () {$0$};
			\node[blue!60!black] at (axis cs: -9.6, -0.45) () {$-1$};
		\end{axis}
	\end{tikzpicture}}
	\caption{}
	\label{fig:polarization}
\end{subfigure}
\hfill
\begin{subfigure}[b]{0.3\linewidth}
    \centerline{\includegraphics[width=\linewidth]{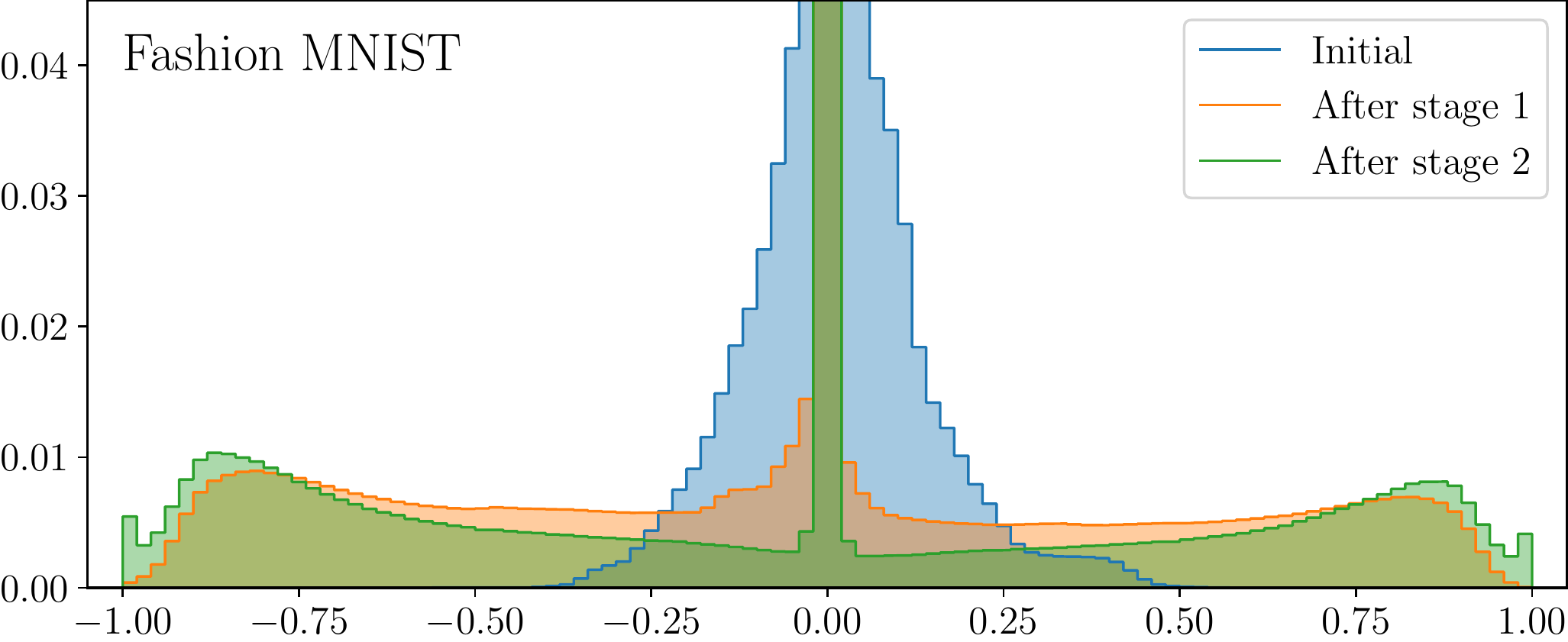}}
    \centerline{\includegraphics[width=\linewidth]{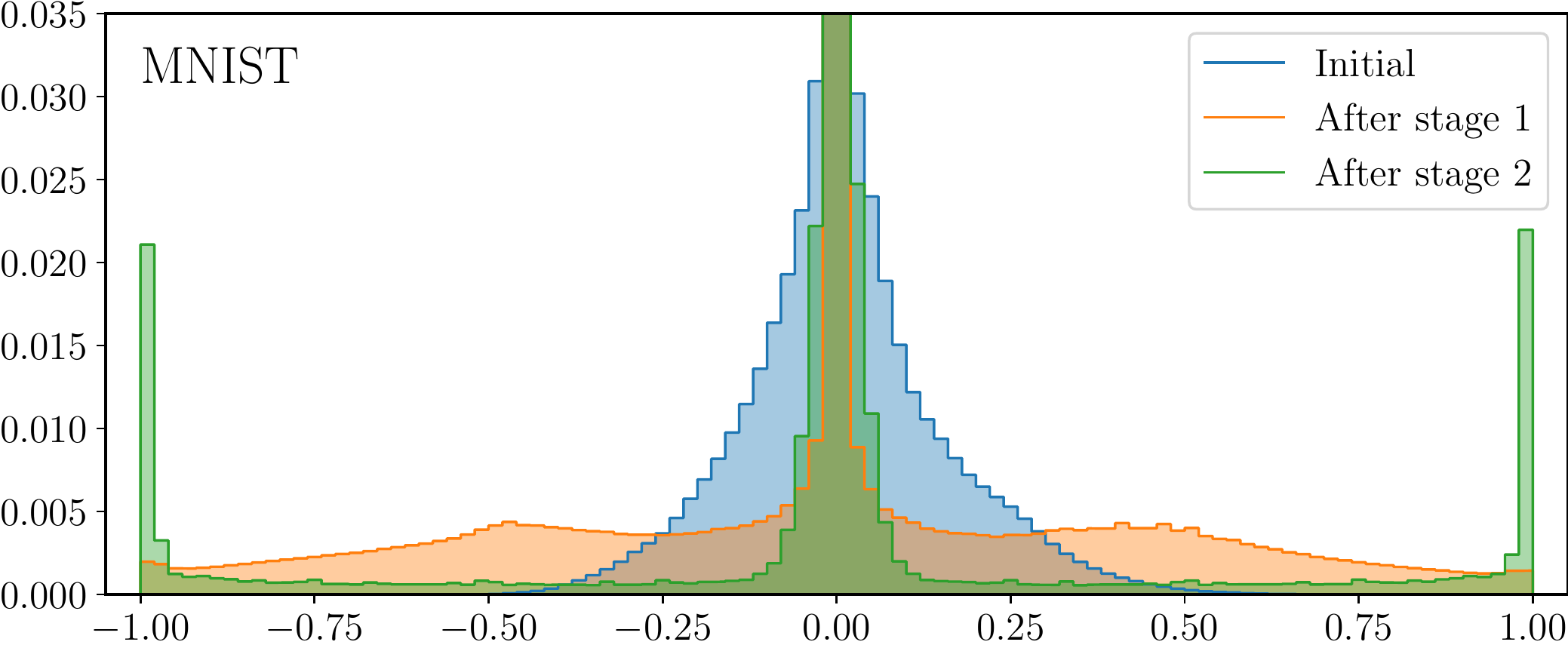}}
  \caption{}
  \label{fig:hist}
\end{subfigure}
\hspace{1pt}
\caption{(a) Activation sparsity (Eq.\ \ref{eq:sparse_act}) alone is not sufficient to achieve robustness: perturbations can ride on top of strongly activated neurons (shaded region). (b) Polarization of neural activity can fully eliminate perturbations. For the shown hypothetical histogram (gray) of $\bw^T\bx/\Vert\bw\Vert_1$, a ternary activation (blue) is effective. (c) Probability distribution of normalized front-end filter outputs $a_k / {\norm{\bw_k}_1} $.}
\end{figure*}
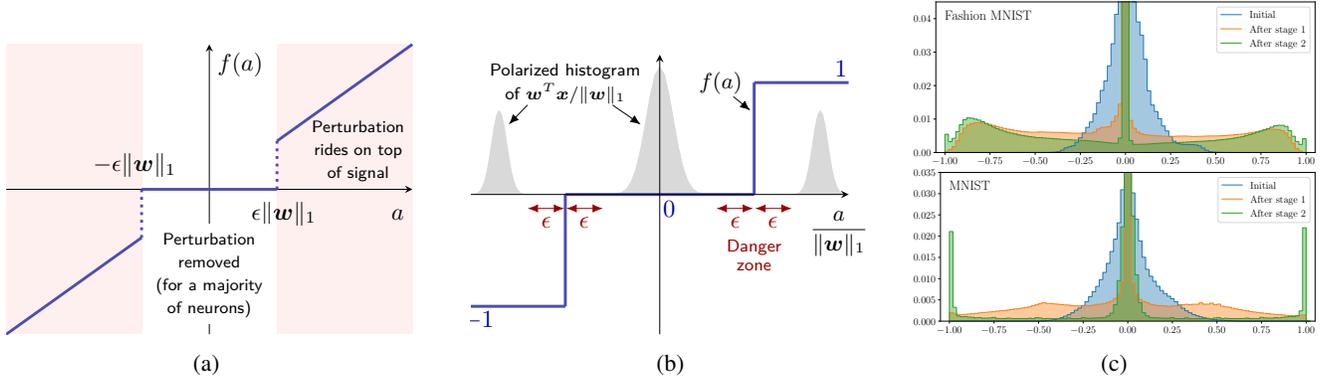

Suppose we have a classifier that takes in inputs $\bx \in \mathbb{R}^N$, and outputs predictions (confidence scores for M classes) $\by \in [0,1]^M$.  Our goal is to defend against malicious inputs of the form $\bx+\be$, where $\be\in \mathbb{R}^N$ is a small perturbation that aims to cause misclassification. Formally, we can describe such adversarial attacks as a maximization problem:
\begin{equation} \label{attacker} 
\underset{\bm{e} \in \mathcal{S}}{\mathrm{max}} \; L(\bm{\theta}, \bm{x}+\bm{e}, \bm{y}_\mathrm{true}),
\end{equation}
\noindent where $L$ is a loss function, $\bm{\theta}$ denotes network weights and biases and $\bm{y}_\mathrm{true}$ is the vector of true labels. The adversary aims to find the perturbation that maximizes $L$, subject to the condition that 
$\be$ is chosen from a set
$\mathcal{S}$ (typically $\ell_p$ bounded). In this paper, we focus on $\ell_{\infty}$ bounded attacks: $\quad \norm{\bm{e}}_\infty \leq \epsilon$ for
an ``attack budget'' $\epsilon > 0$. Furthermore, we assume a ``white box'' attack, in which the adversary has full knowledge of the network structure and weights.\\[-4pt]

\noindent
{\bf Attacks:} State of the art $\ell_{\infty}$ bounded attacks (used in our evaluations) are all based on gradient ascent on the cost function in (\ref{attacker}). The {\it Fast Gradient Sign Method (FGSM)}  \cite{harnessing2015}, computes the perturbation by
\begin{equation} \label{fgsm}
    \bm{e}= \epsilon \cdot \text{sign}(\nabla_{\bm{x}} L(\bm{\theta},\bm{x},\bm{y}))
\end{equation}
An iterative version of FGSM known as the {\it Basic Iterative Method (BIM)} \cite{iterative2017}  finds the perturbation as
\begin{equation}
    \bm{e}_{i+1}= \text{Clip}_{\epsilon} \Big( \bm{e}_{i} + \alpha \cdot \text{sign}(\nabla_{\bm{x}} L(\bm{\theta},\bm{x}+\bm{e}_i,\bm{y})) \Big)
\end{equation}
where $\alpha$ is the step size for each iteration, and $\epsilon$ is the overall $\ell_\infty$ attack budget.
It was noted in \cite{madry} that BIM is a formulation of Projected Gradient Descent (PGD), a well-known method in convex optimization. The PGD attack suggested in \cite{madry} employs  BIM with multiple random starting points sampled from
a uniform distribution in the $\epsilon$ box around the data point.  We term this scheme {\it PGD with Restarts.}\\[-4pt]

\noindent
{\bf Defenses:} Defenders seek to minimize (\ref{attacker}), so that learning in an adversarial setting may be viewed as a minimax game. A number of defense mechanisms have been proposed, only to be defeated by stronger adversaries \cite{carlini2016distillationrefuted,obfuscated}. The current state of the art defense employs retraining with adversarial examples \cite{madry}. 
However, there is no design intuition as to why and how perturbations are being controlled as they flow up the network. It is also computationally intensive, slowing down training by an order of magnitude. 
A more efficient and interpretable line of defenses employ data preprocessing prior to the network. For example, sparsity-based preprocessing was shown to be effective for linear classifiers \cite{isit2018} and neural networks \cite{iclrw2018,tsp2019}.
More recently, \cite{menet} proposed preprocessing by randomly erasing pixels of the image, followed by reconstruction using well-known matrix estimation methods. 
When combined with adversarial training, \cite{menet} achieves state-of-the-art performance on MNIST, CIFAR-10, SVHN, and Tiny-ImageNet datasets.

A number of quantization-based defense methods have been proposed in literature, within the neural network \cite{defensive_quantization,rakin2018defend} and as a front end \cite{qusecnet, discretization_based, limitations_discretization,randomized_discretization}. The key difference in our proposed strategy is that we employ polarization prior to quantization, which enables theoretical guarantees on robustness (Section \ref{sec:rationale}).  
\\[-4pt]

\noindent
{\bf Gradient Masking:} The use of non-differentiable functions or functions with a saturation region can cause state of the art gradient-based attacks to falter.
However, defenses that rely on such ``gradient masking'' are not robust: they are easily circumvented by the attacker, as shown in \cite{obfuscated},
by replacing the non-differentiable function by a differentiable approximation. 

We test our defense using the gradient approximation methods of \cite{obfuscated}, replacing non-differentiable functions with identity in the gradient calculations. We have also performed experiments with other differentiable approximations to our quantization function, but found identity approximation to be the most effective for the attack.

\section{Polarization-based Defense} \label{sec:rationale}

We investigate a defense based on a front end which preprocesses the inputs via a linear transformation followed by a nonlinear activation $f$. 
Following convention, the linear operation of a particular filter is termed a \textit{neuron}. Consider a typical front-end neuron with weights $\bw$, and scalar output $a$. For perturbed input $\bx+\be$ with $\ell_\infty$ bound $\Vert\be\Vert_\infty < \epsilon$, $a$ contains two components: desired signal $\bw^T\bx$, and an output perturbation $\bw^T\be$ that is constrained in magnitude: $\left|\bw^T\be\right| \leq \Vert\be\Vert_\infty \Vert\bw\Vert_1 \leq \epsilon \Vert\bw\Vert_1$ due to H\"{o}lder's inequality. For the defense to be successful, the nonlinearity $f$ must be chosen such that $f(a=\bm{w}^T(\bm{x}+\bm{e})) \approx f(\bw^T\bx)$. 

One design approach is to promote sparse activations by increasing the threshold for neurons to fire, which makes it difficult for a small perturbation to induce firing:
\begin{equation} \label{eq:sparse_act}
f(a) = \begin{cases}
			0, \quad |a| \leq\epsilon \Vert\bw\Vert_1\\
			a, \quad \text{otherwise.}
		\end{cases}
\end{equation}
While this method helps (see \cite{isit2018,iclrw2018,tsp2019} for a similar approach), Fig.\ \ref{fig:sparse_act} shows why it cannot be completely successful. When a neuron resides near the middle of the unshaded region, no perturbation can change the signal output ($f(a)=0$). However, neurons with a strong desired signal component (large $|\bw^T\bx|$) can serve as hosts for the perturbation, allowing it to propagate through the defense.
Hence activation sparsity can only be a part of the solution.

What if we could somehow \textit{polarize} neural activity to obtain well-separated clusters of neurons? Consider for instance the three clusters of activations shown in Fig.\ \ref{fig:polarization}. In such a scenario, we can completely eliminate perturbations by using a quantized nonlinearity (in this case, ternary quantization). Note that it is important for neurons to avoid the ``danger zones'' of width $2\epsilon$ shown in the figure: this ensures that perturbations cannot switch data from one quantization level to the next. These observations are formalized in the following proposition.

\begin{proposition}
Suppose the front end polarizes activations into a multimodal distribution with $L$ clusters, with minimum inter-cluster separation $d > 2\epsilon\Vert\bw\Vert_1$. Let $c_1 < c_2 < \dots c_{L-1}$
denote the midpoints between adjacent clusters.
Then the following $L$-level quantizer (with thresholds at $c_i$)  completely eliminates perturbations with $\ell_\infty$ norm smaller than $\epsilon$:
\begin{equation}\label{eq:quant}
f(a) = \frac{1}{2} \sum_{i=1}^{L-1} \emph{sign} (a-c_i).
\end{equation}
\end{proposition}

\begin{proof}

Since we use a quantizing nonlinearity, perturbations can cause distortion only if the output switches quantization levels. We know that for a perturbation $\be$ with $\ell_\infty$ budget $\epsilon$, the maximum output distortion is $\epsilon\Vert\bw\Vert_1$. Therefore, if clusters are separated by a distance of $2\epsilon\Vert\bw\Vert_1$, perturbations cannot propagate through the defense.
\end{proof}
\vspace{-2pt}

This result motivates a second design approach, where we seek a neural basis in which outputs are well-polarized for clean inputs, with clusters of $\bw^T\bx/\Vert\bw\Vert_1$ separated by at least 2$\epsilon$ as shown in Fig. \ref{fig:polarization}. We can then choose a piecewise constant nonlinearity (Eq.\ \ref{eq:quant}) to eliminate the effect of perturbations.
Equipped with these design principles, we now detail training procedures to learn polarizing bases from data.

\subsection{Implementing a Polarizing Front End}
\label{sec:description}

\begin{figure}[!htbp]
  	\centering
	\begin{tikzpicture}	[
		pre/.style={=stealth',semithick, dotted},
		post/.style={->,shorten >=1pt,>=stealth',semithick, dotted}
		]
	\node[] (x) {$\bx+\be$};
	\node[rect_filter] (conv) [right = 4mm of x] {\footnotesize{\begin{tabular}{c}Polarizing \\[-1.5pt] filters \\[-1.5pt] $[\bw_1, \dots \bw_K]$\\[1pt]\end{tabular}}};
	\node[rect] (normalize) [right = 5mm of conv] {$\frac{1}{\norm{\bm{w}_k}_1}$};
	\node[rect] (quantize) [right = 5mm of normalize] {$Q(\cdot)$};
	\node[rect2] (w) [right = 4mm of quantize] {\small{Classifier}};
	\node[] (z) [below right = -3mm and 1.5mm of normalize] {};
	\node[outer, inner sep =2mm, fit = (conv) (normalize) (quantize), label=above:\small{Front end}] (frontend) {};
	\node[] (out) [right = 2mm of quantize] {
		};
	\draw[->] (x) --node[anchor=south]{} (conv); 
	\draw[->] (conv) --node[anchor=north, label={[label distance=0.35mm]above:{$\ba$}}]{} (normalize);
	\draw[->] (normalize) --node[anchor=north, label={[label distance=0.35mm]above:{$\bz$}}]{} (quantize);
	\draw[->] (quantize) --node[anchor=south]{} (w);
	\end{tikzpicture}
	\caption{Block diagram of front end defense, showing a polarizing filter followed by $\ell_1$ normalization and quantization.}
  	\label{fig:system}
\end{figure}
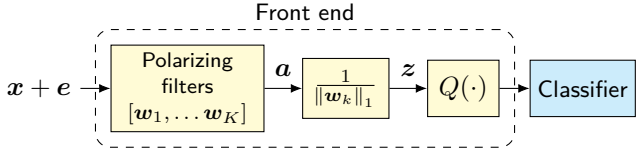

We employ a front end (shown in Fig. \ref{fig:system}) which uses convolutional filters to learn polarized and quantized latent representations of data. For front end neuron $\bw_k$, let $z_k = a_k /\Vert\bw_k\Vert_1$ denote the normalized activation. We seek a multimodal distribution for $\bz$, with clusters separated by at least $2\epsilon$. We achieve this by training with \textit{bump regularizers} $B_1(\cdot)$ and $B_2(\cdot)$ which promote polarization of data. We train in three stages by minimizing the modified loss function:
\begin{equation*}
	\mathcal{L}(\bm{y}, \bm{y}_{\mathrm{true}}, \bm{z}) = \mathcal{L}_{CE}(\bm{y},\bm{y}_{\mathrm{true}}) + \frac{\lambda}{K} \sum_{k=1}^{K} {B(z_k)}.
\end{equation*}  
where $\mathcal{L}_{CE}$ is the cross-entropy loss determined by the true label and outputs of the classifier, $K$ is the number of neurons, $\bz$ is the vector of activations of all neurons $[z_1, \dots z_K]$, $B$ is the regularizer and $\lambda$ is a scaling coefficient. These stages can be described as follows:
\begin{enumerate}
	\item We start by training the polarizer without using any quantization. The front end filters are initialized from a uniform distibution described in \cite{he2015delving}. Due to the random initialization, normalized activations are typically clustered around zero initially (shown in Fig. \ref{fig:hist}). Next we incorporate a bump regularizer $B(\cdot)=B_1(\cdot)$ to drive the normalized activations away from the origin, pushing $\bz$ towards the endpoints $-1$ and $1$. 

	\begin{equation*}
		B_1(z_k) = e^{-{z_k^2}/{2\sigma_1^2}}.
	\end{equation*}
	\item  After achieving a sufficiently even level of distribution throughout the interval $[-1,1]$, we switch to the second bump regularizer $B(\cdot)=B_2(\cdot)$, aimed at pushing the normalized activations away from the quantization thresholds $\pm c$ and polarizing $\bz$ into three clusters centered at $-1$, $0$ and $1$. 
	\begin{equation*}
		B_2(z_k) = e^{-{(z_k-c)^2}/{2\sigma_2^2}} + e^{-{(z_k+c)^2}/{2\sigma_2^2}}.
	\end{equation*}
	\item Now we introduce the quantization function $f_2(\cdot)$ described in Eq.~\ref{eq:ternary}. We also freeze and stop training the filters in the front end, and remove the regularizer. We train the classifier to let the weights adapt to the quantization.
	\begin{equation} \label{eq:ternary}
		f_2(z_k) = 0.5\sign(z_k-c) + 0.5\sign(z_k+c).
  \end{equation}
\end{enumerate}
For testing, we continue using the quantized activation in \eqref{eq:ternary} to eliminate perturbations. Details regarding the choice of parameters such as  $\lambda$, $c$, $\sigma_1$ and $\sigma_2$ are given in Section~\ref{sec:results}.

\begin{figure}[t]
    \centerline{\includegraphics[width=0.85\columnwidth]{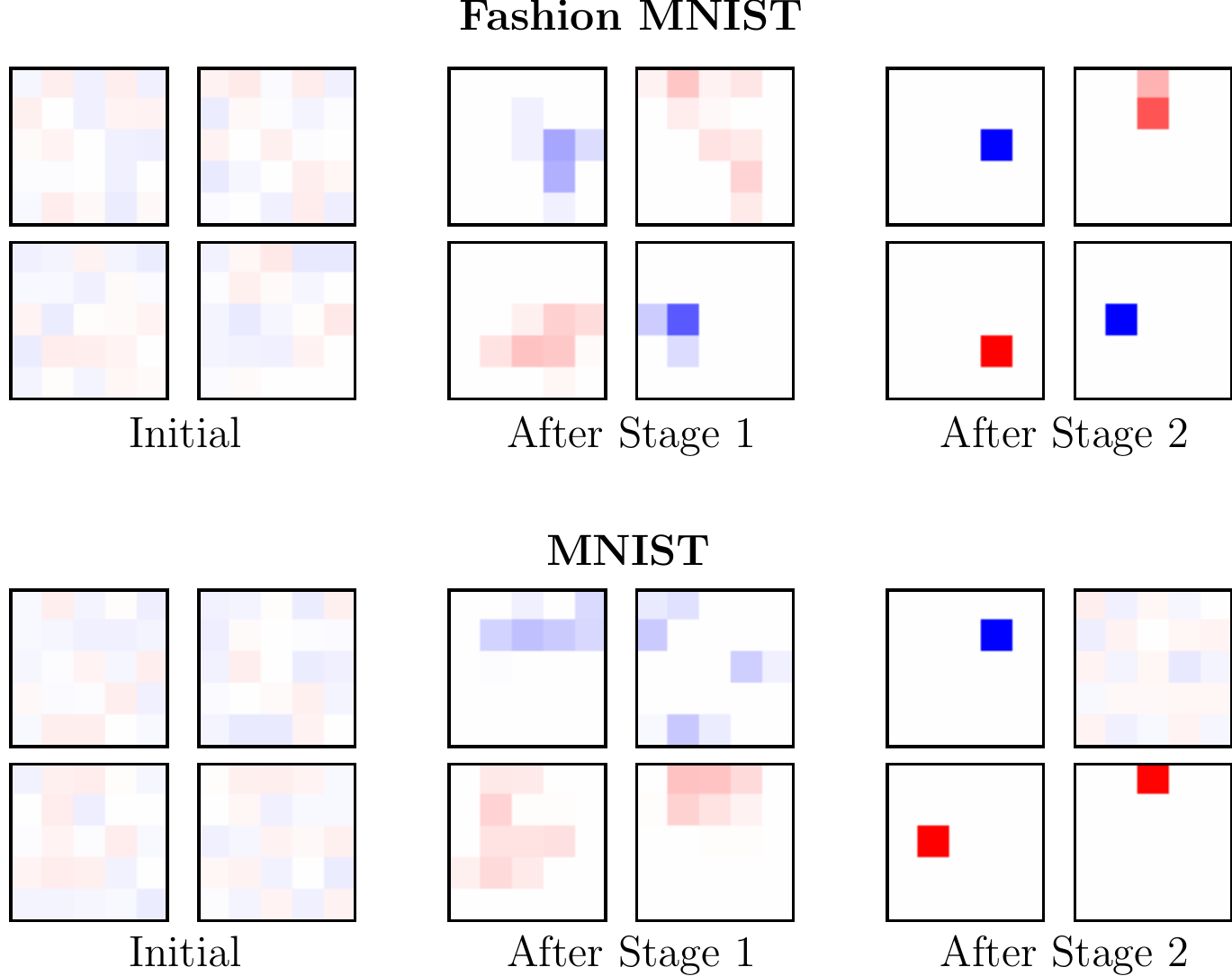}}

    \caption{Typical progression of front-end filters over stages.} 
    \label{fig:filters}

\end{figure}

We find that these three stages suffice for Fashion MNIST and MNIST, but depending on the dataset, one could potentially repeat Stage 2 with an increasing number of clusters until the desired level of polarization is achieved. Fig. \ref{fig:hist} demonstrates the effects bump regularizers have on the distribution of normalized activations.

Fig.~\ref{fig:filters} shows the filters obtained after each stage for MNIST and Fashion MNIST. 
Interestingly, we find that the learnt filters appear similar to pixel bases. This is consistent with the observations in \cite{madry} about first-layer filters learnt by adversarial training on MNIST.

\section{Experiments and Results}
\label{sec:results}

\subsection{Training Details}
\label{ssec:training_detail}

For a fair comparison we use the small convolutional neural network from \cite{madry}, consisting of two convolutional layers and two fully connected layers. Convolutional layers have 32 and 64 number of filters that are 5x5 in size. 
Each convolutional layer is followed by 2x2 maxpooling operation. Every layer except the last uses ReLU activation function. The outputs of the last layer are fed into a softmax function to generate classification probabilities. 
In every run, the model is trained for 20 epochs in each stage for a total of 60 epochs. 
Gradient descent is achieved using the Adam optimizer \cite{adam} with learning rate $10^{-3}$ and default hyperparameters in PyTorch library. 

During training with bump regularizers, stage 1 and stage 2 bump widths are picked to be $\sigma_1=0.35$ and $\sigma_2=0.15$ respectively. 
To make the adaptation of weights smoother, we increase the bump coefficient $\lambda$ linearly from 0 to 1 in each stage, as the stages progress. 
The quantization threshold is chosen to be $c = 0.3$ for Fashion MNIST and $c = 0.5$ for MNIST.
When adversarially training using the methods of Madry et al. \cite{madry} we use 10 restarts and 20 steps in each restart. 
\\[-3pt]

\noindent
{\bf Attack Setup:} We evaluate our defense against the white box attacks described in Section \ref{sec:background}: FGSM, BIM and PGD with Restarts. We use attack budget $\epsilon=0.3$ for MNIST and $\epsilon=0.1$ for Fashion MNIST. In iterative methods, we use step size 
$\alpha=\epsilon/10$. In BIM, we use 20 steps. In PGD we choose the best performing attack from 20 random restarts, with 100 steps in each restart.

\newcommand{\ra}[1]{\renewcommand{\arraystretch}{#1}}
\setlength{\tabcolsep}{0.24em}

\begin{table}[!b]\centering
\ra{1.3}
\begin{footnotesize}
\vskip -0.5em
\caption{Experimental results for different attacks.}
\label{table:results_table}
\begin{tabular}{@{}rcccccccccc@{}}
\\[-3pt]
\toprule
& \multicolumn{4}{c}{Fashion MNIST ($\epsilon=0.1$)} & \phantom{a}& \multicolumn{4}{c}{MNIST ($\epsilon=0.3$)} & \phantom{a}\\ \cmidrule{2-5} \cmidrule{7-10}
&  Clean & FGSM & BIM & PGD && Clean & FGSM & BIM  & PGD \\ \midrule
No defense & \multicolumn{1}{|c}{91.6} & 19.7 & 1.49 & 0.11 && 99.4 &  21.9 &  0.47  & 0.00 \\
Adv. Training  &  \multicolumn{1}{|c}{83.8} & 77.9 & 75.6 & 74.1 && 97.5 & 93.1 & 90.0 & 86.7 \\
Ours  & \multicolumn{1}{|c}{87.3} & 69.4 & 54.9 & 44.5 && 99.1 & 93.2 & 86.5 & 70.8 \\ \bottomrule
\end{tabular}
\end{footnotesize}
\end{table}

\begin{figure}[!t]
    \centerline{\includegraphics[width=0.93\columnwidth]{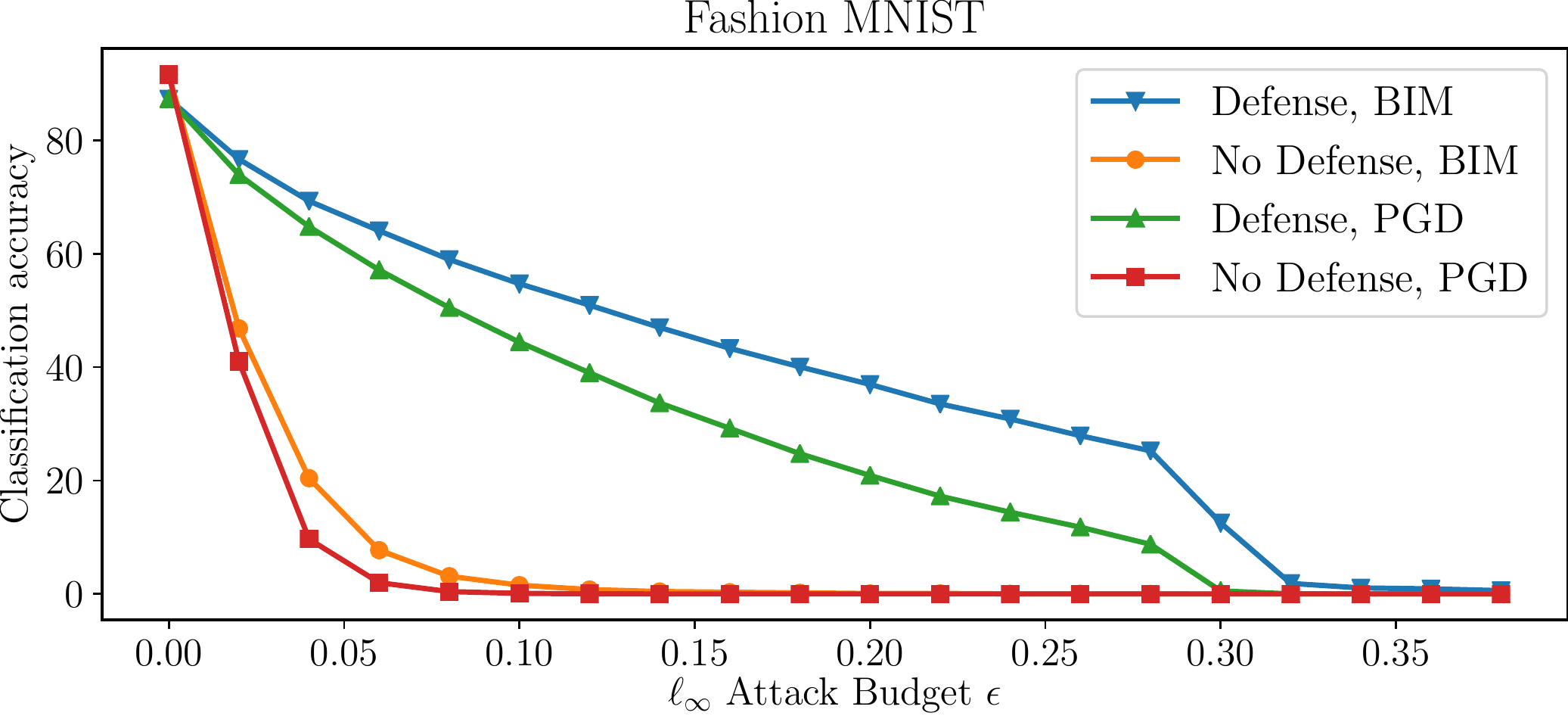}}
    \vspace{.2cm}
    \centerline{\includegraphics[width=0.93\columnwidth]{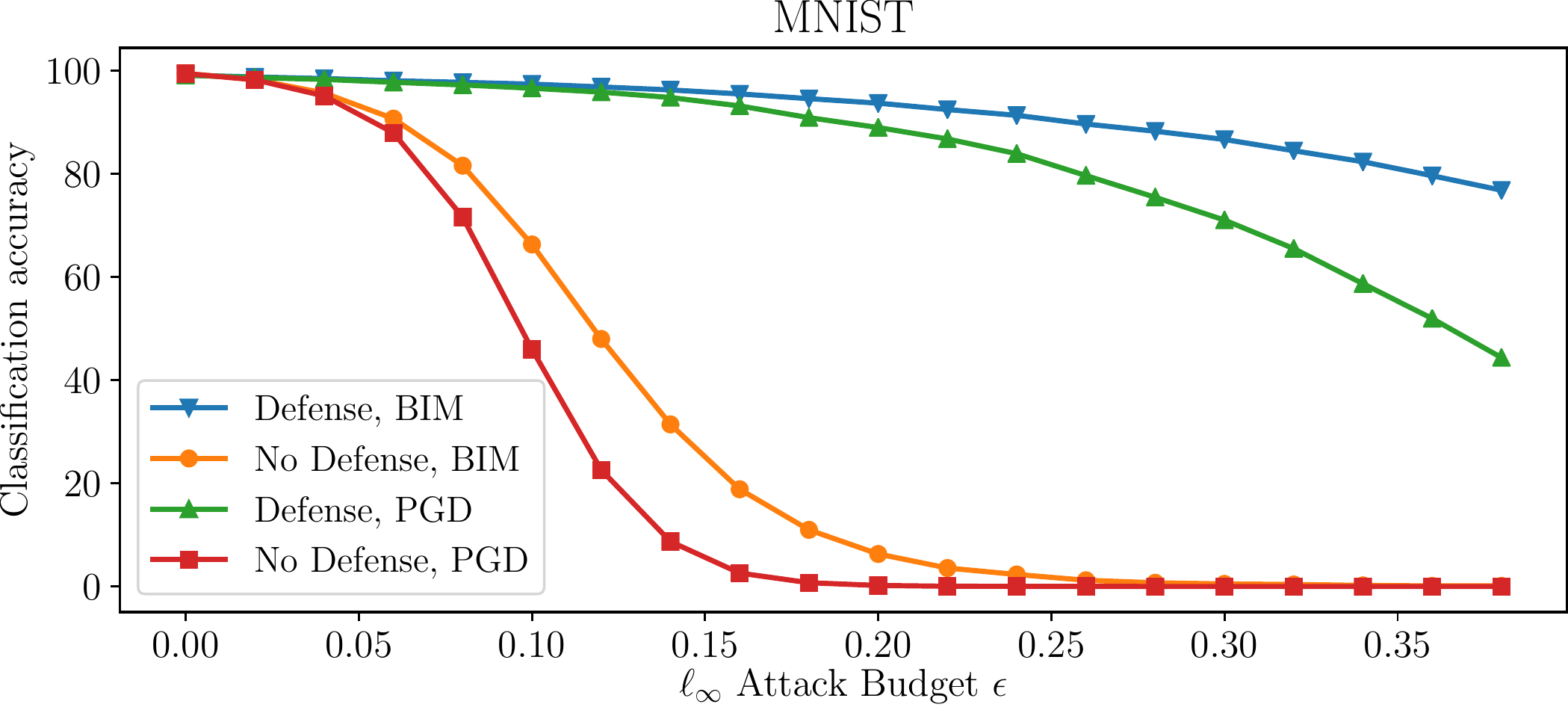}}
    \caption{Classification accuracy versus $\ell_\infty$ attack budget.} 
    \label{fig:epsilon}
\end{figure}

\subsection{Results and Discussion}

Fig. \ref{fig:epsilon} shows the effect of attack budget on accuracy, showing that our front end increases adversarial accuracy across a wide range of $\epsilon$. Table \ref{table:results_table} details our results against different attacks, with a comparison with methods from literature. 

Our defense significantly improves robustness against a variety of attacks, but falls short of the accuracies obtained by adversarial training. This is because perfect polarization is not possible in practice, leading to some leakage of adversarial perturbations through the front end. 
However,  the polarization approach is amenable to interpretation, and provides an avenue for further efforts in systematic bottom-up design. In contrast, empirical experiments are the only means of verifying the efficacy of state of the art adversarial training. 

\section{Conclusions}
\label{sec:discussion}

In this paper, we have shown that polarization is a promising tool for defense against adversarial attacks: when data is perfectly polarized, quantization can provably eliminate perturbations. Our training procedures for learning polarizing bases indicate that pixel bases are effective for polarizing datasets like MNIST and Fashion MNIST, which is consistent with the first-layer filters learnt in adversarially trained models for these datasets \cite{madry}. While we consider a supervised learning framework here, we have also obtained promising results with unsupervised learning of polarizing bases, but omit discussion due to space constraints. Open problems for future work include combining polarizing front ends with nonlinearities within the network in order to provably attenuate attacks as they flow through the network, and obtaining polarizing bases for more complex datasets such as CIFAR and ImageNet.

\section{Acknowledgements}
\label{sec:ack}

This work was supported in part by the Army Research Office under grant W911NF-19-1-0053, and by the National Science Foundation under grant CIF-1909320.

\clearpage
\bibliographystyle{IEEEbib}
\bibliography{ms.bbl}

\begin{thebibliography}{10}

\bibitem{wildpattern2018}
B.~Biggio and F.~Roli,
\newblock ``Wild patterns: {Ten} years after the rise of adversarial machine
  learning,''
\newblock {\em Pattern Recognition}, vol. 84, pp. 317--331, December 2018.

\bibitem{szegedy2013intriguing}
C.~Szegedy, W.~Zaremba, I.~Sutskever, J.~Bruna, D.~Erhan, I.~Goodfellow, and
  R.~Fergus,
\newblock ``Intriguing properties of neural networks,''
\newblock in {\em International Conference on Learning Representations}, 2014.

\bibitem{harnessing2015}
I.~J. Goodfellow, J.~Shlens, and C.~Szegedy,
\newblock ``Explaining and harnessing adversarial examples,''
\newblock in {\em International Conference on Learning Representations}, 2015.

\bibitem{iterative2017}
A.~Kurakin, I.~J. Goodfellow, and S.~Bengio,
\newblock ``Adversarial machine learning at scale,''
\newblock in {\em International Conference on Learning Representations}, 2017.

\bibitem{madry}
A.~Madry, A.~Makelov, L.~Schmidt, D.~Tsipras, and A.~Vladu,
\newblock ``Towards deep learning models resistant to adversarial attacks,''
\newblock in {\em International Conference on Learning Representations}, 2018.

\bibitem{carlini2016distillationrefuted}
N.~Carlini and D.~Wagner,
\newblock ``Towards evaluating the robustness of neural networks,''
\newblock in {\em IEEE Symposium on Security and Privacy}, 2017, pp. 39--57.

\bibitem{obfuscated}
A.~Athalye, N.~Carlini, and D.~Wagner,
\newblock ``Obfuscated gradients give a false sense of security: Circumventing
  defenses to adversarial examples,''
\newblock in {\em International Conference on Machine Learning}, 2018.

\bibitem{isit2018}
Z.~Marzi, S.~Gopalakrishnan, U.~Madhow, and R.~Pedarsani,
\newblock ``Sparsity-based defense against adversarial attacks on linear
  classifiers,''
\newblock in {\em IEEE International Symposium on Information Theory (ISIT)},
  2018.

\bibitem{iclrw2018}
S.~Gopalakrishnan, Z.~Marzi, U.~Madhow, and R.~Pedarsani,
\newblock ``Combating adversarial attacks using sparse representations,''
\newblock in {\em International Conference on Learning Representations (ICLR)
  Workshop Track}, 2018.

\bibitem{tsp2019}
S.~Gopalakrishnan, Z.~Marzi, U.~Madhow, and R.~Pedarsani,
\newblock ``Robust adversarial learning via sparsifying front ends,''
\newblock {\em arXiv preprint arXiv:1810.10625}, 2018.

\bibitem{menet}
Y.~Yang, G.~Zhang, D.~Katabi, and Z.~Xu,
\newblock ``{ME-Net}: {Towards} effective adversarial robustness with matrix
  estimation,''
\newblock in {\em International Conference on Machine Learning}, 2019.

\bibitem{defensive_quantization}
J.~Lin, C.~Gan, and S.~Han,
\newblock ``Defensive quantization: {When} efficiency meets robustness,''
\newblock in {\em International Conference on Learning Representations}, 2019.

\bibitem{rakin2018defend}
A.S. Rakin, J.~Yi, B.~Gong, and D.~Fan,
\newblock ``Defend deep neural networks against adversarial examples via fixed
  and dynamic quantized activation functions,''
\newblock {\em arXiv preprint arXiv:1807.06714}, 2018.

\bibitem{qusecnet}
H.~Ali, H.~Hammad~Tariq, M.~A. Hanif, F.~Khalid, S.~Rehman, R.~Ahmed, and
  M.~Shafique,
\newblock ``{QuSecNets}: {Quantization-based} defense mechanism for securing
  deep neural network against adversarial attacks,''
\newblock in {\em IEEE International Symposium on On-Line Testing and Robust
  System Design}, 2019.

\bibitem{discretization_based}
P.~Panda, I.~Chakraborty, and K.~Roy,
\newblock ``Discretization based solutions for secure machine learning against
  adversarial attacks,''
\newblock {\em IEEE Access}, vol. 7, 2019.

\bibitem{limitations_discretization}
J.~Chen, X.~Wu, V.~Rastogi, Liang Y., and S.~Jha,
\newblock ``Towards understanding limitations of pixel discretization against
  adversarial attacks,''
\newblock in {\em IEEE European Symposium on Security and Privacy}, 2019.

\bibitem{randomized_discretization}
Y.~Zhang and P.~Liang,
\newblock ``Defending against whitebox adversarial attacks via randomized
  discretization,''
\newblock in {\em International Conference on Artificial Intelligence and
  Statistics}, 2019.

\bibitem{he2015delving}
K.~He, X.~Zhang, S.~Ren, and J.~Sun,
\newblock ``Delving deep into rectifiers: Surpassing human-level performance on
  imagenet classification,''
\newblock in {\em Proceedings of the IEEE International Conference on Computer
  Vision}, 2015, pp. 1026--1034.

\bibitem{adam}
D.P. Kingma and J.~Ba,
\newblock ``Adam: A method for stochastic optimization,''
\newblock in {\em International Conference on Learning Representations}, 2015.

\end{thebibliography}

\end{document}